\documentclass[twoside]{article}

\usepackage{tikz}
\usetikzlibrary{positioning}

\usepackage[linesnumbered,ruled,vlined]{algorithm2e}
\usepackage{bm}
\usepackage{natbib}
\usepackage{graphicx}
\usepackage{comment}
\usepackage{enumerate}   
\usepackage{subcaption}
\usepackage{amssymb}
\usepackage{amsmath}

\DeclareMathOperator*{\argmin}{arg\,min}

\usepackage{latexsym}
\usepackage{amsmath}
\usepackage{appendix}
\newtheorem{theorem}{Theorem}
\newtheorem{corollary}{Corollary}[section]

\newtheorem{example}{Example}[section]

\newenvironment{proof}[1][Proof]{\textbf{#1.} }{\
\qquad\qquad\rule{0.5em}{0.5em }}
\usepackage{color,graphicx,lscape,longtable}

\newtheorem{lemma}{Lemma}[section]

\usepackage{times}

\usepackage{aistats2020}
\begin{document}

%

%

\newcommand{\matsuda}[1]{\textcolor{blue}{(MA: #1)}}
\newcommand{\masa}[1]{\textcolor{red}{(MU: #1)}}

\twocolumn[

\aistatstitle{Unified estimation framework for unnormalized models \\
with statistical efficiency}

\aistatsauthor{ Author 1 \And Author 2 \And  Author 3 \And Author 4 }

\aistatsaddress{ Institution 1 \And  Institution 2 \And Institution 3 \And Institution 4} ]
\newcommand{\masa}[1]{\textcolor{red}{(MU: #1)}}

\begin{abstract}
Parameter estimation of unnormalized models is a challenging problem because normalizing constants are not calculated explicitly and maximum likelihood estimation is computationally infeasible. Although some consistent estimators have been proposed earlier, the problem of statistical efficiency does remain. In this study, we propose a unified, statistically efficient estimation framework for unnormalized models and several novel efficient estimators with reasonable computational time regardless of whether the sample space is discrete or continuous. The loss functions of the proposed estimators are derived by combining the following two methods: (1) density-ratio matching using Bregman divergence, and (2) plugging-in nonparametric estimators. We also analyze the properties of the proposed estimators when the unnormalized model is misspecified. Finally, the experimental results demonstrate the advantages of our method over existing approaches.  
\end{abstract}

\section{Introduction}

Unnormalized models are widely used in many settings: Markov random field \citep{BesagJulian1975SAoN}, Boltzmann machines \citep{HintonGeoffreyE.2002TPoE}, models in independent component analysis \citep{HyvarinenAapo2001Ica}, submodular diversity model \cite{Tschiatschek} and generalized gamma distributions \citep{gamma}. When the parametric model is denoted as $p(x;\theta)$, $p(x;\theta)$ is called an unnormalized (aka intractable) model if its normalizing constant $\int p(x;\theta)\mathrm{d}\mu(x)$
cannot be calculated explicitly, or it is difficult to compute in practice. For example, when $\mu$ is a counting measure as in the case of Markov networks and Boltzmann machines, the computational cost increases exponentially with the dimension of the sample space. When $\mu$ is a Lebesgue measure, as in the case of models in independent component analysis or generalized gamma distributions, this cannot be calculated analytically. When we use unnormalized models, we believe that the true data generating process is approximated by the family $\{p(x;\theta)/\int p(x;\theta)\mathrm{d}\mu(x),\theta\in \Theta\}$, where $\Theta$ denotes a parameter space. Unnormalized models $p(x;\theta)$ can be converted to  normalized models by dividing their normalizing constants. However, their explicit form cannot be obtained; therefore, an exact maximum likelihood estimation (MLE) is infeasible. 

Several approaches for the estimation of unnormalized models have been suggested. Two approaches are important. First, noise contrastive estimation (NCE) \citep{noise,noise2,TCL,matsuda2} and contrastive divergence (CD) \citep{HintonGeoffreyE.2002TPoE} rely on sampling techniques, such as importance sampling and Markov Chain Monte Carlo (MCMC). Second, score matching \citep{score,HyvärinenAapo2007Seos,DawidA.Philip2012PLSR,dawid2} and pseudo likelihood \citep{BesagJulian1975SAoN,VarinCristiano2011AOOC,LindsayBruceG.2011IASI} uses a tractable form without the aid of a sampling technique. Generally, the first approach is superior to the second approach in terms of statistical efficiency, whereas the second approach is superior to the first approach in terms of computational efficiency, leaving a tradeoff between computational efficiency and statistical efficiency.

\begin{figure}[!]
    \centering
    \includegraphics[width=8cm]{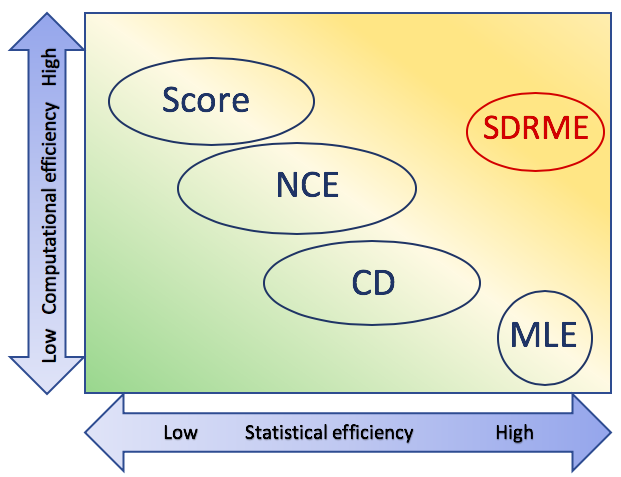}
    \caption{Comparison of methods, where Score stands for score matching, NCE stands for a noise contrastive estimation, CD stands for a contrastive divergence method, and a self density-ratio matching estimator (SDRME) is the proposed estimator. Note that statistically efficient estimators can be constructed in the case of NCE and CD. When the ratio of the auxiliary sample size and original sample size is infinite in NCE, the estimator becomes statistically efficient. However, it is infeasible to implement it in practice. The same argument applies to CD.}
    \label{fig:comparison}
\end{figure}

In the present study, we propose a unified framework for the statistically efficient estimation of unnormalized models and several statistically efficient estimators irrespective of whether the sample space is discrete or continuous. The estimators are defined as the form of M-estimators \citep{VaartA.W.vander1998As} and their loss functions are derived by combining two methods: (1) density-ratio matching using Bregman divergence, and (2) plugging-in nonparametric estimators. These estimators are statistically efficient in the sense that the asymptotic variance is the same as that of the MLE; thus, the proposed estimators are superior to other previously proposed estimators in terms of statistical efficiency. Moreover, the proposed estimators do not rely on any sampling techniques and the evaluation cost of the objective function is $\mathcal{O}(n)$; therefore, they are competitive in terms of computational efficiency. Figure \ref{fig:comparison} illustrates the superiority of our proposed estimators to the other previously proposed estimators. To the best of our knowledge, proposed estimators are the first statistically efficient estimators with evaluation cost $\mathcal{O}(n)$, which works in the continuous sample space.
 
It should be noted that when the sample space is discrete, \cite{takekana2017} proposed an efficient estimator, which can be seen as a special case from our proposed framework. Importantly, it is extended to the case of continuous sample space based on the proposed framework.

\section{Preliminaries}

Our general setting is as follows. Let us consider a situation in which an unnormalized model $p(x;\theta)$ is used, that is, for each $\theta\in \Theta$,
$p(x;\theta)$ is a non-negative function and the normalizing constant defined by the integral
$\int_{\mathcal{X}}p(x;\theta)\mathrm{d}\mu(x)$, is finite. The measure $\mu$ over the sample space $\mathcal{X}$ is a counting measure when the sample space is discrete, and a Lebesgue measure when the sample space is continuous. We refer to it as a baseline measure in this paper.  We also introduce a one-parameter extended model defined by $q(x;\tau)\equiv\exp(-c)p(x;\theta),\,\tau\equiv(c,\theta^{\top})^{\top}$ where $c$ is regarded as a parameter. \masa{changed}

Our aim is to estimate $\theta$ using a set of identically independent distributed (i.i.d) $n$ samples $\{x_{i}\}_{i=1}^{n}$ by assuming that these samples are obtained from the true distribution $\mathrm{F}_{\eta^{*}}$ with density $
\eta^{*}(x)$ with respect to the baseline measure $\mu$. Unless otherwise noted, we assume that the unnormalized model is well-specified, that is, there exists $\theta^{*}$ satisfying $\eta^{*}=\exp(-c^{*})p(x;\theta^{*}),\,\exp(c^{*})=\int p(x;\theta^{*})\mathrm{d\mu}(x)$.
The problem of unnormalized models arises because it is extremely difficult or infeasible to calculate the normalizing constant analytically. In such a case, one should avoid a direct computation of the normalizing constant; therefore, the loss function of MLE cannot be used. In this section, we review the Bregman divergence and the generalized NCE, needed to understand proposed methods.

We summarize frequently used notations. We denote $\mathrm{E}_{*}(\cdot)$ as an expectation under the true density $\eta^{*}(x)$. The notations $\mathrm{Var}_{*}(\cdot)$ and $\tilde{\mathrm{E}}_{*}(\cdot)$ represent variance and empirical analogues. Notation $\mathbb{P}_{n}$ denotes an empirical distribution of $n$ samples from the true distribution $\mathrm{F}_{\eta^{*}}$. We denote $\mathrm{d}\mathbb{P}_n/\mathrm{d}\mu$ as $p_{n}$, evaluation at $\tau$, that is, $|_{\tau=\tau^{*}}$ as $|_{\tau^{*}}$, and $\nabla_{x}$ as the differentiation with respect to $x$. A summary of the notation is provided in a table in the Supplementary materials.

\subsection{Bregman divergence}

Let $\mathbb{R}_{\geq0}$ be a set of non-negative real numbers. We define $\mathcal{F}$ as a collection of non-negative real-valued functions on the sample space $\mathcal{X}$, and we assume that $\mathcal{F}$ is a convex set.
Given a convex function $\psi(u)$ on $\mathcal{F}$, the Bregman divergence \citep{BregmanL.M.1967Trmo,GneitingTilmann2007SPSR,DawidAlexander2014Taao} on $\mathcal{F}\times\mathcal{F}$ is defined as $
 B_{\psi}(u,v) = \psi(u)-\psi(v)-\nabla\psi(v)(u-v)$,
where $\nabla\psi(v)$ is a linear operator defined by $
 \lim_{\varepsilon\rightarrow+0}\left[\{\psi(v+\varepsilon h)-\psi(v)\}/\varepsilon \right]=\nabla\psi(v)(h)$. 
Here, $h$ is a function on $\mathcal{X}$ such that
$v+\varepsilon h\in\mathcal{F}$ holds for arbitrary small $\varepsilon>0$.
The convexity of $\psi(u)$ guarantees the non-negativity of the Bregman divergence. We introduce two kinds of Bregman divergences; one is separable, while the other is non-separable.

The separable Bregman divergence is defined using the function $\psi(u)$:
\begin{align}
\label{eq:potential}
\psi(u)=\mathrm{E}_{*}[f\{u(x)\}],  
\end{align}
where $f:\Rbb_{\geq0}\rightarrow\Rbb$ is a strictly convex function. For the differentiable $f$, the corresponding Bregman divergence $ B_{f}(u,v;\eta^{*})$ between $u$ and $v$ is given as
 \begin{align}
  \label{eqn:separable-Breg}
\mathrm{E}_{*}[f\{u(x)\}-f\{v(x)\}-f'\{v(x)\}\{u(x)-v(x)\}]. 
 \end{align}
 For the strictly convex function $f$, the corresponding $B_{f}(u,v;\eta^{*})$ vanishes if and only if $u=v$ up to a null set with respect to the measure $\eta^{*}(x)\mathrm{d}\mu(x)$.

 \begin{example}
For $f(x)=2x\log{x}-2(1+x)\log(1+x)$, the corresponding $B_{f}(u,v;\eta^{*})$ is known as the Jensen-Shannon divergence. In other cases, for $f(x)=x\log{x}$, we have the Kullback-Liber (KL) divergence. For $f(x)=x^{m}/\{m(m-1)\}$, we get the $\beta$-divergence  \citep{BasuAyanendranath1998RaEE,MurataNoboru2004IGoU}. 
 \end{example}

The Bregman divergence is non-separable if the convex function $\psi(u)$ is not expressed as \eqref{eq:potential}. The pseudo-spherical divergence and the $\gamma$-divergence are examples of non-separable Bregman divergences, and they are commonly used in robust inference \citep{KanamoriTakafumi2014Aida}. The function $\psi(u)$ of the pseudo-spherical divergence is the $\gamma$-norm with $\gamma>1$ under the density $\eta^{*}(x)$, that is, $\|u\|_{\gamma}=\mathrm{E}_{*}\{u(x)^{1+\gamma}\}^{-\gamma/(1+\gamma)}$. The pseudo-spherical divergence $B_{\text{ps}}(u,v;\eta^{*})$ is defined as 
\begin{align}
\label{eq:spherical}
\|u\|_{\gamma}-\frac{1}{\|v\|_{\gamma}^{\gamma-1}}\mathrm{E}_{*}\{ v(x)^{\gamma-1}u(x)\}. 
\end{align}
The pseudo-spherical divergence $B_{\text{ps}}(u,v;\eta^{*})$ vanishes if and only if $u$ and $v$ are linearly dependent. When we apply a log-transformation to each term in  \eqref{eq:spherical}, this becomes a $\gamma$-divergence \citep{FujisawaHironori2008Rpew}, represented as
\begin{align}
 \label{eq:gamma}
 B_{\gamma}(u,v;\eta^{*})&=\frac{1}{\gamma}\log \mathrm{E}_{*}\{u(x)^{\gamma}\}+\frac{\gamma-1}{\gamma}\log \mathrm{E}_{*}\{v(x)^{\gamma}\}-  \\
 &\log \mathrm{E}_{*}\{ v(x)^{\gamma-1}u(x)\} \nonumber.
 \end{align}

\subsection{Generalized noise contrastive estimation}
\label{Generalized NCE}

We review an estimation method for unnormalized models focusing on generalized NCE. A generalized NCE \citep{noise2,hirayama} is proposed. The strategy to estimate $\theta,c$ in $q(x;
\tau)$ is estimating a density ratio between the $q(x;\tau)/a(x)$, where $a(x)$ is a known auxiliary density, by generating samples from the distribution with a density $a(x)$.

Using a set of samples $\{y_{i}\}_{i=1}^{n}$ from the auxiliary distribution with a density $a(y)$ with respect to the baseline measure $\mu$, the estimator $\hat{\tau}_{\text{NC}}$ for $\tau$ is defined as the minimizer of the following function
\begin{align}
\label{eq:loss}
& \frac{1}{n}\sum_{i=1}^{n} r_{q,a}(y_{i};\tau)f'\left\{r_{q,a}(y_{i};\tau)\right\}-  \\ & f\left\{r_{q,a}(y_{i};\tau)\right\}-f'\left\{r_{q,a}(x_{i};\tau)\right\} \nonumber ,
\end{align}
where $r_{q,a}(x;\tau)=q(x;\tau)/a(x)$, $f(x)$ is a strictly convex function and the support of the density $a(x)$ includes the support of $p(x;\theta)$. This estimation is derived from a divergence perspective as follows: let the divergence between the true distribution $\eta^{*}(x)$ and the one-parameter extended model $q(x;\tau)$ be $
B_{f}\left\{r_{\eta^{*},a}(x),r_{q,a}(x); a(x)\right\}$ when $r_{\eta^{*},a}(x)=\eta^{*}(x)/a(x)$. We have $B_{f}\left\{r_{\eta^{*},a}(x),r_{q,a}(x); a(x)\right\}\geq 0$ and $B_{f}\left\{r_{\eta^{*},a}(x),r_{q,a}(x); a(x)\right\}=0 \Leftrightarrow \eta^{*}(x)=q(x;\tau^{*})$. Therefore, the estimation problem of $\tau$ is reduced to a minimization problem of $B_{f}\left\{r_{\eta^{*},a}(x),r_{q,a}(x); a(x)\right\}$ with respect to $\tau$. By subtracting the term not associated with $q(x;\tau)$ from $B_{f}\left\{r_{\eta^{*},a}(x),r_{q,a}(x); a(x)\right\}$, we obtain the term 
\begin{align*}
&-\int f'\left\{r_{q,a}(x)\right\}\eta^{*}(x)\mathrm{d}\mu(x)+ \\
&\int \left[f'\left\{r_{q,a}(x)\right\}r_{q,a}(x)-f\left\{r_{q,a}(x)\right\}\right]a(x)\mathrm{d}\mu(x).  
\end{align*}
The loss function of $\hat{\tau}_{\text{NC}}$, \eqref{eq:loss}, is constructed using an empirical approximation of this term. 

Unless otherwise noted, we hereafter assume the following properties for $f(x)$:
\begin{assumption}
\label{as:1}
Function $f: \Rbb_{+} \to \Rbb$  satisfies the following three properties: strictly convex, third-order differentiable and $f''(1)=1$. 
\end{assumption}
Among $f(x)$ satisfying the above conditions, the estimator when $f(x)=2x\log x-2(1+x)\log(1+x)$ is proven to be optimal from the perspective of asymptotic variance, irrespective of the auxiliary distribution \citep{Uehara}. In this case, the loss function of the estimator becomes
\begin{align}
\label{eq:noise2}
 -\frac{1}{n}\sum_{i=1}^{n} \log \frac{r_{q,a}(x_{i};\tau)}{1+r_{q,a}(x_{i};\tau)}-\frac{1}{n}\sum _{i=1}^{n} \log \frac{1}{1+r_{q,a}(y_{i};\tau)}.
\end{align}
This loss function is identical to the original NCE \citep{noise}. Although it satisfies some aforementioned optimality, the asymptotic variance of the estimator derived from the above loss function is larger than that of MLE. We can also use another type of $f(x)$. For example, when $f(x)=x\log x$, ths loss function is 
\begin{align}
\label{eq:noise_mle}
 -\frac{1}{n}\sum_{i=1}^{n} \log r_{q,a}(x_i;\tau)+\frac{1}{n}\sum _{i=1}^{n} r_{q,a}(y_{i};\tau).
\end{align}
This is reduced to the same from as the one of Monte Carlo MLE \citep{GeyerC1994Otco}:
\begin{align}
 \label{eq:noise_profile}
 -\frac{1}{n}\sum_{i=1}^{n} \log p(x_i;\theta)+\log \left \{\frac{1}{n}\sum _{i=1}^{n} \frac{p(y_{i};\theta)}{a(y_i)} \right \} 
\end{align}
by profiling out $c$ beforehand. The asymptotic variance of Monte Carlo MLE is lager than that of NCE \citep{Riou-DurandLionel2018NceA}. 

We have so far assumed that the sample size of the auxiliary distribution goes to infinity at the same rate as the size of the true distribution when considering an asymptotic regime. On the other hand, when the sample size of the auxiliary density grows faster than the sample size of the true distribution, it is easily proved that Monte Carlo MLE is statistically efficient. However, this asymptotic regime is suggesting that the evaluation cost of the objective function is larger than $\mathcal{O}(n)$, which is the order when MLE can be done exactly. This is problematic because it requires a lot of computational time. Throughout this paper, our goal is finding an efficient estimator such that the evaluation cost of the objective function is $\mathcal{O}(n)$. 

\section{Estimation with self density-ratio matching}
\label{sec:esti}

We propose two types of statistically efficient estimators with reasonable computational time. Our key idea is to match the ratio of the unnormalized model and nonparametrically estimated density using Bregman divergence. We introduce an estimator based on a separable Bregman divergence. Then, we introduce an estimator based on a non-separable Bregman divergence. 

\subsection{Separable case}

We introduce an estimator called self density-ratio matching estimator (SDRME) for $\tau$ as a form of M-estimators:
\begin{align}
\label{eq:breg-separable}
 \hat{\tau}_{\text{s}}=\argmin_{\tau\in \Theta_{\tau}} \mathrm{B}_{f}[h_{1}\{w(x;\tau)\},h_{2}\{w(x;\tau)\};p_{n}],
\end{align}
where $\Theta_{\tau}$ is a parameter space for $\tau$, $w(x)=q(x;\tau)/\hat{\eta}_{n}(x)$, $\hat{\eta}_{n}(x)$ is the nonparametric estimator using an entire set of samples, $q(x;\tau)$ is a one-parameter extended model in Section \ref{Generalized NCE}, $p_{n}=\mathrm{d}\mathbb{P}_{n}/\mathrm{d}\mu$, and $h_{1}(x)$ and $h_{2}(x)$ are functions satisfying conditions mentioned in the next paragraph. The reason why we introduce $h_1,h_2$ is generalizing a result as much as posisble.\masa{changed} More specifically, the loss function is written as
\begin{align}
\label{ea:ob-separable}
\frac{1}{n}\sum_{i=1}^{n} B_{f}\{h_{1}(w_i),h_2(w_i)\},  
\end{align}
where $w_{i}=q(x_{i};\tau)/\hat{\eta}_{n}(x_{i})$. Importantly, it requires only sample order $\mathcal{O}(n)$ calculation. 

When the baseline measure is a counting measure, we use an empirical distribution $p_{n}(x)$ as $\hat{\eta}_{n}(x)$, whereas when the baseline measure is a Lebesgue measure, we use a kernel density estimator as $\hat{\eta}_{n}(x)$. Here, three conditions for $h_{1}(x), h_{2}(x)$ are assumed. 
\begin{assumption}\label{asm:asm}
Functions $h_{1}:\Rbb_{+} \to \Rbb$ and $h_{2}:\Rbb_{+} \to \Rbb$ must be (\ref{asm:asm}\RN{1}) monotonically second-order differentiable increasing functions,  (\ref{asm:asm}\RN{2}) $h_{1}(x)=h_{2}(x)\iff x=1$, and (\ref{asm:asm}\RN{3}) $h'_{1}(1)\neq h'_{2}(1)$.
\end{assumption}
The condition (\ref{asm:asm}\RN{2}) is required for the identification, and (\ref{asm:asm}\RN{3}) is needed to state the asymptotic normality of the estimators. \masa{changed}

This estimator works based on the following intuitive equivalence. By replacing $p_{n}(x)$ and $\hat{\eta}_{n}(x)$ with $\eta^{*}(x)$ in \eqref{eq:breg-separable}, we obtain $
\mathrm{B}_{f}\{h_{1}(w),h_{2}(w);\eta^{*}\}=0  \iff  h_{1}(w)=h_{2}(w) 
 \iff w=1 \iff q(x;\tau)=\eta^{*}(x)$. As explained in Section \ref{investigation}, this estimator is rigorously proven to be consistent and efficient. Several specific choices can be considered as $h_{1}(w)$ and $h_{2}(w)$ as follows. We also explain SDRME graphically in Figure \ref{fig:intuitive}.

\begin{figure}
    \centering
    \includegraphics[width=6cm]{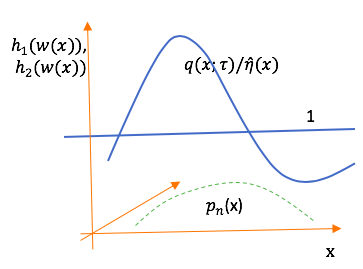}
    \caption{Graphical explanation of SDRME with $h_1(w)=w,\,h_{2}(w)=1$. The objective function is measuring the difference between $q(x;\tau)/\hat{\eta}_n(x)$ and $q(x;\tau^{*})/\eta(x)=1$ utilizing a Bregman divergence with the density $p_{n}(x)$. }
    \label{fig:intuitive}
\end{figure}

\begin{table}[!]
    \centering
        \caption{Comparison between Generalized NCE and SDRME in Eg 3.1. Both methods are seen as matching ratio with 
    $\tau$ (ratio 1) to a target ratio (ratio 2). Here, efficiency means statistical efficiency.}
    \begin{tabular}{ccc} \\
     & Generalized NCE & SDRME in Eg 3.1 \\ \hline
    Ratio 1 & $q(x;\tau)/a(x)$ & $q(x;\tau)/\hat{\eta}_{n}(x)$ \\ 
    Ratio 2 & $q(x;\tau^{*})/a(x)$ & $q(x;\tau^{*})/\eta(x)=1$ \\
    Evaluation cost & $\mathcal{O}(n)$ & $\mathcal{O}(n)$ \\
    Efficiency  & No & Yes \\ 
    \end{tabular}
    \label{tab:comparison}
\end{table}

\begin{example}[As an extension of Generalized NCE ]
\label{ex:gene}
Consider a case where $h_{1}(w)=1$ and $h_{2}(w)=w$. The loss function becomes 
\begin{align*}
\frac{1}{n}\sum_{i=1}^{n}
 \big\{-f'(w_i)  + w_i f'(w_i) -f(w_i) \big\}.
\end{align*}
This is considered to be a natural extension of generalized NCE as in Table \ref{tab:comparison} because when we replace $a(x)$ with $\hat{\eta}_{n}(x)$, and $y_{i}$ with $x_{i}$ in \eqref{eq:loss}, the loss function \eqref{eq:loss} is the same as the one above. Especially, when $f(x)=x\log x$, the loss functions corresponding \eqref{eq:noise_mle} and \eqref{eq:noise_profile} are  
\begin{align}
  & -\frac{1}{n}\sum_{i=1}^{n}\log q(x_{i};\tau)+ {\frac{1}{n}\sum_{i=1}^{n}\frac{q(x_{i};\tau)}{\hat{\eta}_{n}(x_{i})}}, \label{eq:sdrme_mle}\\
& -\frac{1}{n}\sum_{i=1}^{n}\log p(x_{i};\theta)+\log\left \{ {\frac{1}{n}\sum_{i=1}^{n}\frac{p(x_{i};\theta)}{\hat{\eta}_{n}(x_{i})}}\right \}. \nonumber
\end{align}

We can also consider a broader class of estimators by setting $h_{1}(w)=w^{\alpha}$ and $h_{2}(w)=w^{\beta},\,(\alpha\neq\beta)$. This class includes the above as special cases.
\end{example}

\subsection{Non-separable case}
\label{sec:nonsep}

Similar to the separable Bregman divergence case, the pseudo-spherical divergence $B_{\text{ps}}$ and the $\gamma$-divergence $B_{\gamma}$ also provide statistically efficient estimators for unnormalized models. Following the analogy of the separable case when $h_{1}(w)=w^{\alpha}, 
,
,h_{2}(w)=w^{\beta}\,(
\alpha \neq \beta)$, suppose that $B_{\text{ps}}(w^\alpha, w^{\beta};\eta^{*})=0$ holds. Then, $w^\alpha$ should be proportional to $w^{\beta}$ because of the property of pseudo-spherical divergence. This means that $w(x)$ is a constant function.
When $w(x;\theta)$ is $p(x;\theta)/\eta(x)$ and $\eta(x)$ is close to  $\eta^{*}(x)$,
$p(x;\theta)$ should be close to $\eta^{*}(x)$ up to the constant factor.
This implies that the parameter $\theta$ can be estimated using pseudo-spherical divergence. 
Replacing $\eta^{*}(x)$ with a nonparametric estimator, SDRME with non-separable divergence $\hat{\theta}_{\text{ns-ps}}$ is obtained by  
\begin{align*}
 \argmin_{\theta \in \Theta} B_{\text{ps}}(w(x;\theta)^{\alpha/\gamma}, w(x;\theta)^{\beta/\gamma}; p_{n}),
\end{align*}
and  $w(x;\theta)={p(x;\theta)}/{\hat{\eta}_{n}(x)}$, under the condition $\alpha\neq \beta$. Then, the loss function is 
\begin{align*}
\left( \sum_{i=1}^{n} w^{\alpha}_{i}\right)^{\frac{1}{\gamma}}- \left(\sum_{i=1}^{n}w^{\beta}_{i}\right)^{(1-\gamma)/\gamma} \sum_{i=1}^{n}w^{\delta}_{i},    
\end{align*}
where $\delta=(\alpha+\beta(\gamma-1))/\gamma, w_{i}=p(x_{i};\theta)/\hat{\eta}_{n}(x_{i})$. 
By taking a logarithm of each term as in \eqref{eq:gamma}, we can construct a loss function corresponding to $\gamma$-divergence. This is equal to 
$B_{\gamma}(w^{\alpha},w^{\beta};p_{n})$:
\begin{align}
\label{eq:tk}
\frac{1}{\gamma}\log \sum_{i=1}^{n} w^{\alpha}_{i}+\frac{\gamma-1}{\gamma}\log \sum_{i=1}^{n}w^{\beta}_{i}-\log \sum_{i=1}^{n}w^{\delta}_{i}.
\end{align}
We define estimator $\hat{\theta}_{\text{ns-}\gamma}$ as a minimizer of the above function with respect to $\theta$. 

Two things should be noted. First, compared with the case of separable divergence, the unnormalized model $p(x;\theta)$ is directly used instead of a one-parameter extended model $q(x;\tau)=\exp(-c)p(x;\theta)$. This is due to the scale--invariance property of pseudo-spherical divergence \citep{KanamoriTakafumi2014Aida,KanamoriTakafumi2015Reuh}. 
Second, when the baseline measure is a counting measure, Takenouchi and Kanamori (2017) have proposed an estimator that is defined as a minimizer of the following function with respect to $\theta$,
\begin{align*}
&\frac{1}{\gamma}\log \sum_{x\in \mathcal{X}} c_x^{1-\alpha}p(x;\theta)^{\alpha}+\frac{\gamma-1}{\gamma} \log \sum_{x\in \mathcal{X} }c_x^{1-\beta}p(x;\theta)^{\beta} \\
&-\log \sum_{x\in \mathcal{X}}c_x^{1-\delta}p(x;\theta)^{\delta},
\end{align*}
where $c_{x}=n_x/n$, $n_{x}$ is a sample number taking the value of $x$. This loss function is essentially the same as \eqref{eq:tk} by modifying the form of summing. The case was only considered when the sample space is discrete. However, it can be generalized to the case where the sample space is continuous, using our new unified perspective. For simplicity, hereafter, we assume $\delta=0$ to eliminate the third term in \eqref{eq:tk}. This restriction is also reasonable to obtain the convexity as seen in Appendix \ref{sec:convexity}.

\section{Theoretical investigation of self density-ratio matching estimator}
\label{investigation}

We prove that the asymptotic variance of estimators $\hat{\theta}_{\text{s}}$ and $\hat{\theta}_{\text{ns-}\gamma}$ is identical to that of MLE. We utilize the property in which our estimators take the form of Z-estimators with infinite dimensional nuisance parameters \citep{VaartA.W.vander1998As,BolthausenErwin2002LoPT}. For the proofs, refer to Appendix.  

\subsection{Efficiency in the separable case}
\label{separablecase}

First, we discuss the case when the divergence is separable. The estimator $\hat{\tau}_{\text{s}}$ based on the separable divergence is defined as the minimizer of the following function $
n^{-1}\sum_{i=1}^{n} B_f\{h_1(w_i),h_2(w_i)\}$, where $w_{i}=q(x_{i};\tau)/\hat{\eta}_{n}(x_{i})$ and $\hat{\eta}_{n}(x)$ is a nonparametric density estimator using an entire sample. 

When $\hat{\eta}_{n}(x)$ was equal to $\eta^{*}(x)$, this estimator $\hat{\tau}_{\text{s}}$ would be regarded as the solution to 
$\tilde{\mathrm{E}}_{*}[\phi(x;\tau,\eta^{*})]=0$, where $\phi(x;\tau,\eta)$ is 
\begin{align*}
    & f[h_{1}\{w(x)\}]-f[h_{2}\{w(x)\}]- \\
    &f'[h_{2}\{w(x)\}]\left [h_{1}\{w(x)\}-h_{2}\{w(x)\}\right ],
\end{align*}
and $w(x)=q(x;\tau)/\eta(x)$, by differentiating the loss function with respect to $\tau$. Here, the moment condition $\mathrm{E}_{*}\{\phi(x;\tau,\eta^{*})|_{\tau^{*}}\}=0$ holds. This condition guarantees that the estimator converges to $\tau^{*}$. \masa{changed} However, this includes the unknown term $\eta^{*}(x)$. By replacing $\eta^{*}(x)$ with the nonparametric estimator $\hat{\eta}_{n}(x)$, the estimator $\hat{\tau}_{\text{s}}$ is still regarded as a Z-estimator. Specifically, the estimator $\hat{\tau}_{\text{s}}$ is constructed by solving the equation $\tilde{\mathrm{E}}_{*}\{\phi(x;\tau,\hat{\eta}_{n})\}=0 $. When the sample space is discrete, the consistency holds as follows. Before that, throughout this paper, we assume the following condition.
\begin{assumption}
The model is $q(x;\tau)$ is $C^{2}$-function with respect to $\tau$. The parameter space $\Theta_\tau$ is compact and $\tau^{*}$ is in the interior of  $\Theta_\tau$. The equation $q(x;\tau)=\eta(x)$ holds if and only if $\tau=\tau^{*}$. 
\end{assumption}
All of conditions are common conditions used in MLE \citep{VaartA.W.vander1998As}. 

\begin{theorem}[Consistency in discrete space] 
$\hat{\tau}_{\text{s}}\stackrel{p}{\rightarrow}\tau^{*}$ .
\end{theorem}

Next, we show the asymptotic normality of the estimator $\hat{\tau}_{\text{s}}$ when the sample space is discrete. 

\begin{theorem}[Asymptotic normality in discrete space ]
\label{thm:4.2}
When the sample space is discrete,  assume that (2a) the following matrix $\Omega = \mathrm{E}_{*}(\nabla_{\tau}\log q\nabla_{\tau^{\top}}\log q|_{\theta^{*}})$ 
is non-singular, (2b) the second order derivative of the map $\eta \to\phi(x;\tau,\eta)$ is uniformly bounded around in a neighborhood of $\eta^{*}$, then we have 
\begin{align*}
&\sqrt{n}(\hat{\tau}_{\text{s}}-\tau^{*})=\Omega^{-1}\mathbb{G}_{n}\left\{\nabla_{\tau}\log q(x;\tau)|_{\tau^{*}}\right\} +\mathrm{o}_{p}(1),\, \\
&\sqrt{n}(\hat{\tau}_{\text{s}}-\tau^{*})\stackrel{d}{\rightarrow}\mathcal{N}(0,\Omega^{-1}). 
\end{align*}
\end{theorem}
These assumptions originate from Theorem 6.18. in \cite{BolthausenErwin2002LoPT}. Assumption (2b) is required to control the remainder term in the proof. It is commonly used to state an asymptotic normality in MLE \citep{VaartA.W.vander1998As}. 

The variance estimator for $\hat{\tau}_{\text{s}}$ is easily constructed from Theorem \ref{thm:4.2}. Finally, we prove that $\hat{\theta}_{\text{s}}$ in $\hat{\tau}_{\text{s}}=(\hat{c}_{s},\hat{\theta}_{\text{s}})$ is equivalent to MLE in terms of asymptotic variance. 

\begin{corollary}
\label{cor:score}
When the sample space is discrete, we have 
\begin{align*}
    \sqrt{n}(\hat{\theta}_{\text{s}}-\theta^{*})\stackrel{d}{\rightarrow} \mathcal{N}(0,\mathfrak{I}_{\theta^{*}}^{-1}),
\end{align*}
where $\mathfrak{I}_{\theta^{*}}$ is the Fisher information matrix at $\theta^{*}$ of the normalized model, that is, $\mathrm{Var}_{*}\{S(x;\theta^{*})\}$, where
$S(x;\theta)= \nabla_{\theta}\left\{\log p(x;\theta)-\log \int p(x;\theta)\mathrm{d}\mu(x)\right\}$. 
\end{corollary}

Next, we investigate asymptotic behavior when the sample space is continuous. We use the kernel density estimator as a nonparametric estimator for $\eta^{*}(x)$. Note that any nonparametric estimators can also be applied. Assume that $\eta^{*}(x)$ belongs to a \Holder\ class of smoothness $\nu$ \citep{KorostelevAlexander;KorostelevaOlga2011MS:A}. The kernel density estimator is constructed as $\hat{\eta}_{n}(x) = (n\kappa)^{-d_{x}}\sum_{i=1}^{n}K\left\{(x_{i}-x)/\kappa\right\}$,
where $\kappa$ denotes a bandwidth, $K$ denotes a $d_{x}$-dimensional kernel, and $d_{x}$ denotes a dimension of $x$ \citep{SilvermanB.W.1986Defs}. The overall error $\|\hat{\eta}_{n}-\eta^{*}\|_{\infty}$ is $\mathrm{O}_{p}((\log n/n)^{1/2}\kappa^{-d_{x}/2}+\kappa^{\nu})$ by choosing high-order kernel \citep{FAN1}. By selecting the order of bandwidth correctly, we have $\|\hat{\eta}_{n}-\eta^{*}\|_{\infty}=\mathrm{O}_{p}((\log n/n)^{-\frac{\nu}{2\nu+d_{x}}})$ \citep{stone}.

From here, we analyze the asymptotic behavior of estimator $\hat{\tau}_{\text{s}}$ when the sample space is continuous. We conclude that the estimator is still efficient. 
\begin{theorem}[Asymptotic normality in continuous space ]
When the sample space is continuous, 
\label{thm:4.4}
under the conditions used in Theorem \ref{thm:4.2} and (2c): $\nu/2>d_x$, (2d): $\int \|\nabla_{\tau}\log q(x;\tau)\|_{\tau^{*}}\mathrm{d}\mu(x)$ is finite, (2e): there is $\epsilon>0$ such that $\mathrm{E}_{*}\{\sup_{\|u\|<\epsilon}\| \nabla_{\tau} \log q(x+u;\tau)|_{\tau^{*}} \|^{4} \}<\infty$, then, $\hat{\theta}_{\text{s}}$ and $\hat{\theta}_{\text{s}}$ are consistent and 
\begin{align*}
\sqrt{n}(\hat{\tau}_{\text{s}}-\tau^{*})\stackrel{d}{\rightarrow}\mathcal{N}(0,\Omega^{-1}),\,
\sqrt{n}(\hat{\theta}_{\text{s}}-\theta^{*})\stackrel{d}{\rightarrow}\mathcal{N}(0,\mathfrak{J}^{-1}_{\theta^{*}}),
\end{align*}
where $\Omega$ is defined in Theorem \ref{thm:4.2}.
\end{theorem}

Here, assumption (2c) is introduced to control a remainder term. In other words, this condition is stating that the convergence rate of $\hat{\eta}_n$ is $\mathrm{o}_{p}(n^{-1/4})$. This is a mild assumption to state an asymptotic normality so that the reminder term in Taylor expansion is negligible \masa{changed}. Assumptions (2d) and (2e) are introduced following Theorem 8.11. in Newey \& Mcfadden (1994).

\subsection{Efficiency in the non-separable case}

We consider an asymptotic analysis of estimator $\hat{\theta}_{\text{ns-}\gamma}$ with the $\gamma$-divergence. 
When $\mu$ is a counting measure, by differentiating \eqref{eq:tk} with respect to $\theta$ and multiplying by $-\gamma/\alpha$, we get $S_{\alpha,\beta}(x;\theta)$:
\begin{align*}
&\int \frac{\{\nabla_{\theta} \log p(x;\theta)\} w(x;\theta)^{\beta}}{\int w(x;\theta)^{\beta}\mathrm{d\mathbb{P}_{n}}(x)}\mathrm{d\mathbb{P}_{n}}(x)- \\
&\int \frac{\{\nabla_{\theta} \log p(x;\theta) \} w(x;\theta)^{\alpha}}{\int w(x;\theta)^{\alpha}\mathrm{d\mathbb{P}_{n}}(x)}\mathrm{d\mathbb{P}_{n}}(x),
\end{align*}
where $w(x)=p(x;\theta)/\hat{\eta}_{n}(x)$. Importantly, compared with the case in Section \ref{separablecase}, $p(x;\theta)$ is used in $w(x)$ instead of $q(x;\tau)$ because of the scale invariant property of $\gamma$-divergence. The estimator $\hat{\theta}_{\text{ns-}\gamma}$ is defined as the solution to $S_{\alpha,\beta}(x;\theta)=0$. The validity of the estimator is based on the relation $0 = T_{\alpha,\beta}(x;\theta)|_{\theta^{*}}$,
where $T_{\alpha,\beta}(x;\theta)$ is a term replacing $\hat{\eta}_{n}(x)$ with $\eta^{*}(x)$ and $\mathbb{P}_{n}$ with $F_{\eta^{*}}$ in $S_{\alpha,\beta}(x;\theta)$. Actually, the estimator $\hat{\theta}_{\text{ns-}\gamma}$ can be seen as a Z-estimator with infinite and finite-dimensional nuisance parameters, that is, the solution to $\mathrm{\tilde{E}}_{*}[U_{\alpha,\beta}(x;\theta,c_{1},c_{2},\hat{\eta}_{n})]=0$, where $U_{\alpha,\beta}(x;\theta,c_{1},c_{2},\eta)$:
\begin{align*}
\begin{bmatrix}
\nabla_{\theta}\log p(x;\theta) \left \{\frac{ p(x;\theta)^{\beta}}{\exp(c_{1})}\eta(x)^{-\beta} - \frac{ p(x;\theta)^{\alpha}}{\exp(c_{2})}\eta(x)^{-\alpha} \right\}\\
 \exp(c_{1})-p(x;\theta)^{\beta}\eta(x)^{-\beta} \\
\exp(c_{2})-p(x;\theta)^{\alpha}\eta(x)^{-\alpha}
\end{bmatrix}.
\end{align*}

The validity of the estimator is based on the moment condition $0 = \mathrm{E}_{*}\{U_{\alpha,\beta}(x;\theta,c_{1},c_{2},\eta)|_{\theta^{*},c_{1}^{*},c_{2}^{*},\eta^{*}}\}$,
where $\exp(c_{1}^{*})=\exp(c^{*})^{\beta}$ and $\exp(c_{2}^{*})=\exp(c^{*})^{\alpha}$. Note that $\theta$ is a parameter of interests, and $c_{1}$, $c_{2}$ and $\eta$ are nuisance parameters. 
We can derive the asymptotic results as in Section \ref{separablecase}. We conclude that  $\hat{\theta}_{\text{ns-}\gamma}$ is an efficient estimator. 

\begin{theorem} When the sample space is discrete, under the conditions of Theorem \ref{thm:4.2}, we have $\sqrt{n}(\hat{\theta}_{
\text{ns-}\gamma}-\theta^{*})\stackrel{d}{\rightarrow}\mathcal{N}(0,\mathfrak{J}_{\theta^{*}}^{-1})$.
\label{thm:ps-gamma}
When the sample space is continuous, under conditions of Theorem \ref{thm:4.4}, we have $\sqrt{n}(\hat{\theta}_{\text{ns-}\gamma}-\theta^{*})\stackrel{d}{\rightarrow}\mathcal{N}(0,\mathfrak{J}_{\theta^{*}}^{-1})$.
\end{theorem}

\subsection{Convexity}

The convexity is important for optimization. 
Here, we consider the convexity of loss functions.  
Suppose that the model is expressed by unnormalized exponential models, $q(x;\tau)=\exp(\tau^{\top}\xi(x))$, where $\xi(x)$ is a basis function and the corresponding basis function for $c$ is $-1$. Regarding separable estimators $\hat{\tau}_{\text{s}}$ in Example \ref{ex:gene}, we can find sufficient 
conditions to ensure the convexity of loss functions as Theorem 
\ref{thm:concavity}. For the specific examples of $f(x)$, see Appendix \ref{sec:convexity}. 

\begin{theorem}
\label{thm:concavity}
Suppose that $f(z)$ satisfies the inequality 
\begin{align*}
(2z-1)f''(z)+z(z-1)f'''(z)\geq 0    
\end{align*}
for arbitrary $z>0$. Then, 
the loss function of the estimator $\hat{\tau}_{\text{s}}$ in Example \ref{ex:gene} is convex in $\tau$. 
\end{theorem}

\subsection{Misspecified case}

We have assumed that the model includes true density. We can also consider a misspecified case, showing that the behavior of the proposed estimators associated with KL divergence \eqref{eq:sdrme_mle} is asymptotically the same as that of MLE. This implies that similarly to MLE, the proposed estimator converges to the parameter that minimizes the KL-divergence between the model and the true distribution and its asymptotic variance is the same, even when the model is misspecified. Specifically, we have the following theorem. 

\begin{theorem}\label{thm:mis}
Under certain regularity conditions, we have
\begin{align*}
&\sqrt{n}(\hat{\tau}_{\text{s}}-\tau^{*})=\Omega_{1m}^{-1}\mathbb{G}_{n}\left\{\nabla_{\tau}\log q(x;\tau)|_{\tau^{*}}\right\} +\mathrm{o}_{p}(1),\\
&\sqrt{n}(\hat{\tau}_{\text{s}}-\tau^{*})\stackrel{d}{\rightarrow} \mathcal{N}(0,\Omega_{1m}^{-1}\Omega_{2m}\Omega_{1m}^{-1}),
\end{align*}
where $\tau^{*}=(c^{*},\theta^{*})$ is a value such that 
\begin{align*}
    \exp(c^{*})=\int p(x;\theta^{*})\mathrm{d}\mu(x),\,0 =\mathrm{E}_{*}\{S(x;\theta)\},
\end{align*}
and
\begin{align*}
\Omega_{1m}&=-\mathrm{E}_{*}\left[\left\{1-\frac{q(x;\tau)}{\eta^{*}(x)}\right\}\nabla_{\tau^{\top}}\nabla_{\tau}\log q(x;\tau)|_{\tau^{*}}\right] \\ 
&+\mathrm{E}_{*}\left\{\frac{q(x;\tau)}{\eta^{*}(x)}\nabla_{\tau}\log q(x;\tau)\nabla_{\tau^{\top}}\log q(x;\tau)|_{\tau^{*}}\right\}, \\
\Omega_{2m}&=\mathrm{Var}_{*}\left\{\nabla_{\tau}\log q(x;\tau)|_{\tau^{*}}\right\}.
\end{align*}
\end{theorem}

Two implications are observed in this Theorem \ref{thm:mis}. First, when the model includes the true distribution, i.e., $\eta^{*}(x)=q(x;\tau^{*})$, this theorem is reduced to Theorem \ref{thm:4.2}. Second, the resulting form of ${\Omega}_{1m},{\Omega}_{2m}$ has a form similar to terms appeared in the asymptotic result of the MLE estimator when the model is normalized. For detail, see Appendix \ref{sec:mis}.

\section{Numerical experiments}
\label{experiments}

Here we present several examples to illustrate the performance of the proposed procedure, and demonstrate that the asymptotic variance of the proposed estimators is the same as that of MLE. We ran simulations in the settings of restricted Boltzmann machines, submodular diversity models, generalized gamma distributions, and misspecified Poisson models. We used the following package for kernel density estimation \citep{npnp}. We also used 6-th order kernel, and the bandwidth was selected by cross validation based on the likelihood. We compare the following estimators:
\begin{itemize}
    \item \textbf{MLE}: Estimator by MLE. 
     \item \textbf{NCE}: Estimator by NCE \citep{noise}. The sample size of the auxiliary distribution is set as the original sample size.  
          \item \textbf{s-KL, s-Chi, s-JS}: Proposed estimators, i.e., SDRME with a separable divergence $\hat{\theta}_{\text{s}}$. When $f=x\log x$, denote \textbf{s-KL}. When $f=0.5x^{2}$, denote \textbf{s-Chi}. When $f=2x\log x-2(1+x)\log(1+x)$, denote \textbf{s-JS}.
    \item \textbf{ns-$\gamma$}: SDRME with the non-separable $\gamma$-divergence, $\hat{\theta}_{\text{ns-}\gamma}$. Regarding the choice of $\alpha,\beta$, see each each section. We select several $\alpha,\beta$ following an experiment section \citep{KanamoriTakafumi2017GclB}  and show a result when the performance is good. 
\end{itemize}
We do not compare proposed estimators with a score matching type estimators because the superiority about statistical efficiency of NCE over score matching is already shown in \citep{noise}. 

\subsection{Restricted Boltzmann Machine (RBM)}
\label{sec:rbm}

The RBM has the parameter $W\in\Rbb^{d_v\times d_h}$. The joint probability of the RBM with the visible nodes $\bm{v}\in\{+1,-1\}^{d_v}$ and hidden nodes $\bm{h}\in\{+1,-1\}^{d_h}$ is 
$P(\bm{v},\bm{h};W)\propto e^{\bm{v}^T W \bm{h}}$ and the marginal probability of $\bm{v}$ is $P(\bm{v};W)\propto\prod_{k=1}^{d_v}\cosh\{(\bm{v}^TW)_k\}$. 
The unnormalized model of the RBM is thus expressed as $q(\bm{v};\tau) = e^{-c}\prod_{k=1}^{d_v}\cosh\{(\bm{v}^TW)_k\}$ with the parameter $\tau=(c,W)$. 

\begin{table}[t]
 \caption{Monte Carlo mean of the KL divergence between the true probability and estimated
 probability scaled by sample size, $n\cdot \mathrm{KL}(\eta^*(\bm{v}),P(\bm{v};\widehat{W}))$, in RBM. Parenthesis indicates the standard deviation. 
 The computational time (seconds) is measured per each iteration when the sample size is $4000$.}
 \label{tbl:RBM_KL}
\centering
$\dim\bm{v}=10, \dim\bm{h}=2$, iteration: 20\\ 
\begin{tabular}{ccccc}
 $n$ &       s-KL &  ns-$\gamma$  & NCE       & MLE\\ \hline
1000&12.8(3.56)& 13.9(3.23)& 19.7(5.31)& 15.5(4.05)  \\
2000&12.4(4.17)& 13.1(5.13)& 17.3(5.68)& 12.8(4.35)  \\
4000&14.3(5.77)& 14.1(4.68)& 18.0(7.35)& 14.5(6.38)  \\ \hline
Time&0.35 & 0.21 & 2.25 & 0.24                       \\ \hline \\ 
\end{tabular}
\vspace{.5pt}

$\dim\bm{v}=18, \dim\bm{h}=2$, iteration: 20\\ 
\begin{tabular}{ccccc}
 $n$ &       s-KL &    ns-$\gamma$ &  NCE       & MLE\\ \hline
1000& 18.7(4.03)&  21.2(4.57)&  76.7(87.9)& 30.5(4.49)\\
2000& 21.5(3.29)&  23.0(3.34)&  51.9(23.3)& 31.2(5.70)\\
4000& 25.9(8.97)&  25.4(8.38)&  38.8(9.25)& 30.0(6.31)\\ \hline
Time&      2.28 &       0.79 &       5.60 &     43.5  \\ \hline 
\end{tabular}

\end{table}

We compared three estimators: \textbf{s-KL}, \textbf{ns-$\gamma$}, \textbf{NCE} and \textbf{MLE}. 
The parameters in \textbf{ns-$\gamma$} were set to $\alpha=0.01, \beta=-1$ and $\gamma=1.01$. 
In low dimensional models, the MLE is feasible because the normalized constant is accessible in practice. 
For \textbf{s-KL} and \textbf{ns-$\gamma$}, we incorporated the sample-based regularization to make the estimator stable. 
For the empirical distribution of the data $\hat{\eta}_n(\bm{v})$, the mixture model $(1-1/n)\hat{\eta}_n(\bm{v})+u_n(\bm{v})/n$ 
is used as the non-parametric estimator of $\eta(\bm{v})$, where $u_n(\bm{v})$ is the empirical distribution of 
$n$ samples generated from the uniform distribution over $\{+1, -1\}^{d_v}$. The additional term $u_n$ is expected to work as a regularization. In the NCE, the auxiliary distribution was defined as the uniform distribution, and the sample size 
from $a(y)$ was set to $5 n$. 

Table~\ref{tbl:RBM_KL} shows Monte Carlo man, and the standard deviation of the KL divergence between the true probability and 
estimated probability scaled by sample size. For each iteration, the new parameter $W$ is generated as the true parameter. 

We confirmed that the proposed methods, \textbf{s-KL} and \textbf{ns-$\gamma$}, were comparable to the MLE while they did not suffer from the computational burden of the normalization constant.
We observed that the accuracy of the NCE is lower than the other methods. The accuracy of the NCE would be improved using larger samples from the auxiliary distribution, while the computational cost increases.

\subsection{Submodular diversity model}
\label{sec:sub}

For applications such as recommendation systems and information summary, several types of probabilistic submodular models have been developed to model the diversity of item sets.
Among them, \cite{Tschiatschek} have proposed the FLID (Facility LocatIon Diversity) model, which is a probability distribution over subsets $S$ of $\{ 1,\cdots,V \}$. 
Specifically, FLID is defined as 
\begin{align*}
P(S; u,w) \propto \exp \left\{ \sum_{i \in S} u_i + \sum_{d=1}^L (\max_{i \in S} w_{i,d} - \sum_{i \in S} w_{i,d}) \right\},
\end{align*}
where $u_i$ and $w_i = (w_{i,1},\cdots,w_{i,L})$ represent the quality and latent embedding vector of the $i$-th item, respectively ($i=1,\cdots,n$).
Since the computation of the normalization constant of FLID is prohibitive, \cite{Tschiatschek} proposed to estimate this model by using NCE. 

We compared \textbf{s-KL}, \textbf{ns-$\gamma$} and  \textbf{NCE}. The parameters in \textbf{ns-$\gamma$} were set to $\alpha=-0.01, \beta=0.99$ and $\gamma=1.01$.  
We generated samples from the FLID model with $L=2$ and $V=12$.
Here, each entry of $u$ and $w$ were sampled independently from the uniform distribution on $[0,1]$.
For the auxiliary distribution in \textbf{NCE}, we used the product distribution following \cite{Tschiatschek}. 
Table~\ref{tab:submodular} presents the Monte Carlo mean and standard error of the KL divergence between the true density and estimated density; additionally, it presents the computation time of each estimator. These results indicate the significant superiority of \textbf{s-KL} to \textbf{NCE} in terms of statistical efficiency with reasonable computational time. We also observe that the performance of \textbf{s-KL} is more stable than that of \textbf{ns-$\gamma$} in this case. 

\begin{table}
    \centering
    \caption{Monte Carlo mean of the KL divergence between the true density and estimated density, scaled by sample size in a submodular diversity model. The computational time (seconds) is measured per each iteration when $n=2 \times 10^5$.}
    \begin{tabular}{rccc}
    $n$     &   s-KL  & ns-$\gamma$  &   NCE \\ \hline
    $5 \times 10^4$  & $36.4 (7.3)$ & $46.6 (5.9)$  & $44.4 (4.0)$ \\
    $1 \times 10^5$  & $21.5 (4.9)$  & $46.4 (4.2)$  & $37.5 (7.8)$ \\
    $2 \times 10^5$ & $16.9 (7.6)$ & $69.3 (7.0)$ &  $35.9 (20.9)$ \\ \hline 
Time & 4911 & 2020 & 9827 \\ \hline 
    \end{tabular}
    \label{tab:submodular}
\end{table}

\subsection{Generalized gamma distribution}

Here, we consider a distribution with the following unnormalized density $
P(x;\theta_1,\theta_2)\propto \exp(-\theta_{1}x^{2})x^{\theta_{2}}\mathrm{I}(x>0)$, when the baseline measure is the Lebesgue measure, which is referred to as a generalized gamma distribution \citep{gamma}. We set the true value at $(\theta_{1},\theta_{2})=(1.3,1.3)$. 

We compared three estimators: \textbf{s-KL}, \textbf{ns-$\gamma$} and \textbf{NCE}. The parameters in \textbf{ns-$\gamma$} were set to $\alpha=-0.01, \beta=0.99$ and $\gamma=1.01$. Unlike Sections \ref{sec:rbm} and \ref{sec:sub}, we used a kernel density estimator for \textbf{s-KL} and \textbf{ns-$\gamma$}, and a half--normal distribution for \textbf{NCE} as an auxiliary distribution. The Monte Carlo mean of the mean square errors is presented in Table \ref{tab:expo}. This result demonstrates the significant superiority of \textbf{s-KL} and \textbf{ns-$\gamma$} over \textbf{NCE} in terms of statistical efficiency with reasonable computational time even when the sample space is continuous. 

\begin{table}
    \centering
    \caption{Monte Carlo mean of mean square errors scaled by sample size in a generalized gamma distribution. The computational time (seconds) is measured per each iteration when $n=2000$.}
    \begin{tabular}{rccc}
 $n$     &   s-KL  & ns-$\gamma$  &   NCE \\ \hline
 500    & 68.2(10.3)  & 77.6(9.3)  & 250.3(64.0) \\
 1000  & 67.9(7.7)  & 76.3(5.3)  &  240.7(69.7) \\
 2000 & 68.3(5.4) & 75.3(4.6) &  246.1(43.5)\\ \hline 
Time & 1.3 & 1.3 & 0.5 \\ \hline 
    \end{tabular}
    \label{tab:expo}
\end{table}

\section{Conclusion}

We have proposed self density-ratio matching estimators. Importantly, proposed estimators are as statistically efficient as MLE without calculating normalizing constants, regardless of whether the sample space is discrete or continuous. In addition, they do not rely on any sampling techniques. Among the several estimators, we recommend using {\bf s-KL} in Section \ref{experiments} for practical purposes because its experimental performance is stable, its loss function is convex, and it is seen as a projection regarding KL divergence, even when the model is misspecified. 


\bibliographystyle{chicago}
\bibliography{paper-ref}

\newpage 
\appendix
\newpage 

\onecolumn

\section{Notation}

\begin{center}
\begin{tabular}{ l l}
$n$ & Total sample    \\
$\mu$ & Baseline measure \\
$p(x;\theta)$ & Unnormalized Model    \\ 
$\tilde{p}(x;\theta)$ & Normalized model  \\
$c$ & Normalizing constant parameter    \\ 
$\tau$ & $(c,\,\theta^{\top})^{\top}$ \\
$q(x;\tau)$ & One-parameter extended model $\exp(-c) p(x;\theta)$  \\
$\Theta$ & Parameter space for $\theta$ \\
$\Theta_{\tau}$ & Parameter space for $\tau$\\
$\eta^{*}(x)$ & True density \\
$\hat{\eta}_{n}(x)$ & Nonparametric estimator \\
$\mathrm{F}_{\eta^{*}}$ & True distribution  \\
$p_{n}$ & Empirical density \\
$\mathrm{E}_{*}$  & Expectation under true distribution \\
$\tilde{\mathrm{E}}$  & Expectation under empirical distribution \\
$\mathrm{Var}_{*}$ & Variance under true distribution \\
$\nabla_{x}$ & Differentiation with respect to $x$\\
$\mathbb{P}_{n}$ & Empirical distribution of $n$ samples from $\mathrm{F}_{\eta^{*}}$ \\
$\mathbb{G}_{n}$ & Empirical process $\sqrt{n}(\mathbb{P}_{n}-\mathrm{F}_{\eta^{*}})$ \\
$\mathfrak{I}_{\theta}$ & Fisher information matrix for $\theta$ \\
$|_{\tau^{*}}$ & the value at $\tau=\tau^{*}$ \\
$\mathcal{N}(A,B)$ & Normal distribution with mean $A$, variance $B$ \\
$L^{2}(\mathrm{F}_{\eta^{*}})$ & $\mathrm{L}^{2}$-space with the underlying distribution $\mathrm{F}_{\eta^{*}}$ \\
$\mathcal{X}$ & Sample space  \\
$B_{f}(u,v)$ & Bregman divergence based on $f$ between $u$ and $v$ \\
$K$ & Kernel \\
$\hat{\tau}_{\text{s}}$ & Self density-ratio matching estimator with a separable divergence. \\ & Note that it is equal to  $(\hat{c}_{\text{s}},\hat{\theta}_{\text{s}})$\\
$\hat{\tau}_{\text{ns-}\gamma}$ & Self density-ratio matching estimator with a $\gamma$-divergence \\  
$\hat{\tau}_{\text{ns-ps}}$ & Self density-ratio matching estimator with a pseudo spherical divergence \\  
$\|\cdot \|$ & Euclidean norm \\
$\|\cdot \|_{\infty}$ & $l_{\infty}$ norm 
\end{tabular}
\end{center}

\newpage 

\section{Convexity}
\label{sec:convexity}

We see specific examples of $f(x)$, satisfying the above conditions in Theorem \ref{thm:concavity}.

\begin{example}
 For the  functions $f(z)=z\log{z}$ and $f(z)=2z\log{z}-2(1+z)\log(1+z)$, we can confirm the conditions in Theorem \ref{thm:concavity}. However, the function $f(z)=0.5z^2$ does not meet the above conditions.
 In the same way, we can find that the function $f(z)=z^m/\{m(m-1)\}$ with a natural number $m\geq 2$
 does not meet the conditions. 
\end{example}

We have a similar result for non-separable estimators. As for the estimator with $\gamma$-divergence, the loss function is convex if the equality $\delta=0$ holds.

\section{Asymptotics under misspecification}
\label{sec:mis}

We have assumed that the model includes true density. In this section, we consider a misspecified case, showing that the behavior of the proposed estimators associated with KL divergence is asymptotically the same as that of MLE when $f(x)=x
\log x$. This implies that similarly to MLE, the proposed estimator converges to the parameter that minimizes the KL-divergence between the model and the true distribution, even when the model is misspecified. 

Before analyzing the proposed estimators, we review a misspecified case where the model can be normalized properly. The MLE under the misspecified model is equivalent to finding the closest model to the true distribution regarding KL divergence \citep{white1}. 
The MLE estimator $\hat{\theta}_{\text{MLE}}$ converges to the value maximizing the function $\theta\to \mathrm{E}_{*}\{\log p(x;\theta)-\log\int p(x;\theta)\mathrm{d}\mu(x)\}$. We denote this value as $\theta^{*}$. The value $\theta^{*}$ satisfies the equation $\mathrm{E}_{*}\{S(x;\theta)\}=0$, where $S(x;\theta)$ is 
\begin{align*}
    \nabla_{\theta}\left\{\log p(x;\theta)-\log \int p(x;\theta)\mathrm{d}\mu(x)\right\}.
\end{align*}
It is well--known that the estimator $\hat{\theta}_{\text{MLE}}$ has the following asymptotic property, that is, $\sqrt{n}(\hat{\theta}_{\text{MLE}}-\theta^{*})$ converges weakly to the normal distribution with mean $0$ and variance
\begin{align*}
\mathrm{E}_{*}\{\nabla_{\theta^{\top}}S(x;\theta)|_{\theta^{*}}\}^{-1}\mathrm{Var}_{*}\{S(x;\theta)|_{\theta^{*}}\}\mathrm{E}_{*}\{\nabla_{\theta^{\top}}S(x;\theta)|_{\theta^{*}}\}^{-1}    
\end{align*}

The term $\mathrm{E}_{*}\{\nabla_{\theta^{\top}}S(x;\theta)|_{\theta^{*}}\}$ is  
\begin{align*}
  &\mathrm{E}_{*}\left[\left\{1-\frac{\tilde{p}^{*}(x)}{\eta^{*}(x)}\right\}\nabla_{\theta^{\top}}\nabla_{\theta}\log p(x;\theta)|_{\theta^{*}}\right]  +\mathrm{E}_{*}\left\{\frac{\tilde{p}^{*}(x)}{\eta^{*}(x)}\nabla_{\theta}\log p(x;\theta)\nabla_{\theta^{\top}}\log p(x;\theta)|_{\theta^{*}}\right\} \\
   &- \mathrm{E}_{*}\left\{\frac{\tilde{p}^{*}(x)}{\eta^{*}(x)}\nabla_{\theta}\log p(x;\theta)|_{\theta^{*}}\right\} \mathrm{E}_{*}\left\{\frac{\tilde{p}^{*}(x)}{\eta^{*}(x)}\nabla_{\theta^{\top}}\log p(x;\theta)|_{\theta^{*}}\right\},
\end{align*}
where  
\begin{align*}
\tilde{p}(x;\theta)=p(x;\theta)/\int p(x;\theta)\mathrm{d}\mu(x), 
\end{align*}
$\tilde{p}^{*}(x)=\tilde{p}(x;\theta^{*})$. We also have  
\begin{align*}
\mathrm{Var}_{*}\{S(x;\theta)|_{\theta^{*}}\}=\mathrm{Var}_{*}\{\nabla_{\theta} \log p(x;\theta)|_{\theta^{*}}\}.
\end{align*}

Next, consider the asymptotic behavior of $\hat{\theta}_{\text{s}}$ in \eqref{eq:breg-separable} when the model is misspecified. We assume $f(x)=x\log x$, as in Example \ref{ex:gene}. In this case, the estimator $\hat{\theta}_{\text{s}}$ converges in probability to $\theta^{*}$, which satisfies the equation $\mathrm{E}_{*}[S(x;\theta)]=0$.
When $f(x)$ is not $x\log x$, a similar result can be obtained. However, the limits of estimators no longer converge to the same $\theta^{*}$. With these settings, we have the following theorem. 
\begin{theorem}
\label{thm:5.1}
Under regularity conditions as in Theorem \ref{thm:4.2}, we have
\begin{align*}
&\sqrt{n}(\hat{\theta}_{\text{s}}-\theta^{*})=\Omega^{{\dagger}^{-1}}_{1m}\mathbb{G}_{n}\left\{\nabla_{\theta}\log p(x;\theta)|_{\tau^{*}}\right\} +\mathrm{o}_{p}(1),\\
&\sqrt{n}(\hat{\theta}_{\text{s}}-\theta^{*})\stackrel{d}{\rightarrow}\mathcal{N}(0,{\Omega^{{\dagger}^{-1}}_{1m}}{\Omega^{\dagger}}_{2m}{\Omega^{{\dagger}^{-1}}_{1m}}).
\end{align*}
The specific forms of $\Omega^{\dagger}_{1m}$ and $\Omega^{\dagger}_{2m}$ are 
\begin{align*}
\Omega^{\dagger}_{1m} 
&=\mathrm{E}_{*}\left[\left\{1-\frac{q(x;\tau)}{\eta^{*}(x)}\right\}\nabla_{\theta^{\top}}\nabla_{\theta} \log p(x;\theta)|_{\tau^{*}}\right]+
\mathrm{E}_{*}\left\{\frac{q(x;\tau)}{\eta^{*}(x)}\nabla_{\theta}\log p(x;\theta)\nabla_{\theta^{\top}}\log p(x;\theta)|_{\tau^{*}}\right\}\\
&-\mathrm{E}_{*}\left\{\frac{q(x;\tau)}{\eta^{*}(x)}\nabla_{\theta}\log p(x;\theta)|_{\tau^{*}}\right \}
\mathrm{E}_{*}\left\{\frac{q(x;\tau)}{\eta^{*}(x)}\nabla_{\theta^{\top}}\log p(x;\theta)|_{\tau^{*}} \right\},
\end{align*}
 and
 \begin{align*}
    \Omega^{\dagger}_{2m}=\mathrm{Var}\{\nabla_{\theta}\log p(x;\theta)|_{\theta^{*}}\}.
\end{align*}
\end{theorem}

Two implications are observed in Theorem \ref{thm:5.1}. First, when the model includes the true distribution, i.e., $\eta^{*}(x)=q(x;\tau^{*})$, this theorem is reduced to Theorem \ref{thm:4.2}. Second, the resulting form of ${\Omega^{\dagger}}_{1m},{\Omega^{\dagger}}_{2m}$ has a form similar to terms appeared in the asymptotic result of the MLE estimator when the model is normalized. 

\subsection{Misspecified Poisson model}

Here, we examine the bahavior of each estimator when the model is misspecified. We assume unnormalized parametric models $P(x;\theta) \propto\exp(\theta)^x/x!,x\in \mathbb{N}_{\geq 0}$ based on Poisson distributions. We consider two scenarios based on the true distribution (well-specified case) $\exp(-2.0)2.0^{x}/x!$ and, (misspecified case) $0.5\exp(-2.0)2^{x-0.2}
/(x-0.2)!+0.5\exp(-1.0)/(x-1.2)!$. 

We compared five estimators: \textbf{s-KL}, \textbf{s-Chi}, \textbf{s-JS}, \textbf{ns-$\gamma$} and \textbf{MLE}. The Monte Carlo mean and the standard error of KL divergence between true density and estimated density are presented in Table \ref{tab:well}.
This experiment reveals that the performance of each estimator significantly varies in the misspecified case, but not in the well-specified case. It is indicated that \textbf{s-KL} is preferable in terms of the KL divergence because it has a performance similar to that of MLE, even when the model is misspecified. 

\begin{table}
   \caption{Monte Carlo mean and standard error of the KL divergence between the true density and estimated density scaled by sample size in a Poisson model. Parenthesis indicates a standard error.}
    \centering
    well-specified case  \\
    \begin{tabular}{rcccccc}
 $n$    & s-KL  & s-Chi  &  s-JS & ns-$\gamma$ & MLE \\ \hline
 1000  & 0.26  & 0.26 & 0.27 & 0.26 & 0.26 \\
 &(0.03) & (0.03) & (0.04) & (0.03) & (0.03)\\ 
 2000 & 0.25  & 0.26 & 0.25 & 0.25 &  0.25 \\ 
 &(0.03) & (0.04) & (0.04) & (0.03) & (0.03)\\ 
 \hline 
    \end{tabular} 
    \\
    \vspace{.5pt}
misspecified case \\
    \begin{tabular}{rcccccc}
 $n$     &    s-KL  & s-Chi  &  s-JS & ns-$\gamma$  & MLE  \\ \hline
 1000  & 5.6 & 6.0 &  7.3 & 6.2 & 5.5  \\
  &(0.4) & (0.7) & (1.4) & (0.7) & (0.2) & \\ 
 2000 & 11.1 &  11.8 & 14.6 & 12.1 & 10.9 \\ 
 &(0.4) & (0.9) & (1.9) & (0.7) & (0.2)\\  
 \hline
    \end{tabular}
    \label{tab:well}
\end{table}

\section{Proof of Theorems} 

\begin{proof}[Proof of  Theorem 4.1]
Use Theorem 5.11 in \cite{BolthausenErwin2002LoPT}.  We check two conditions; (1a) $\phi(x;\tau,\hat{\eta}_{n})$ belongs to Glivenko Canteli class,(1b) for any $\epsilon>0$, $\inf_{\{\tau:\|\tau-\tau^{*})\|>\epsilon\}}\|\mathrm{E}_{*}\{\phi(x;\tau,\eta^{*}\})\|$. Regarding the first condition, we check in the proof of Theorem \ref{thm:4.2}. Assumption (1b) is verified by the following two conditions: (1c) $\phi(x;\tau,\eta^{*})$ is continuous with respect to $\tau$, (1d) $\mathrm{E}_{*}\{\phi(x;\tau,\eta^{*})\}=0 \iff \tau=\tau^{*}$. The condition (1c) immediately holds assuming that $\theta\to p(x;\theta)$ is continuous. When the identification condition of the model $q(x;\tau_{1})=q(x;\tau_{2})\iff \tau_{1}=\tau_{2}$ holds, (1d) is verified because $\mathrm{E}_{*}[\phi(x;\tau,\eta^{*})]=0 \iff q(x;\tau)=q(x;\tau^{*}) \iff \tau =\tau^{*}$ \citep{Uehara}.
\end{proof}

\begin{proof}[Proof of  Theorem 4.2]

First, under Assumptions 1-3 and (2a), we can check the following conditions;
\begin{itemize}
    \item $(\tau,\eta)\to \phi(x;\tau,\eta)$ is continuous in an $L^{2}$ space $L^{2}(F_{\eta^{*}})$ at $(\tau^{*},\eta^{*})$
    \item $\{\phi(x;\tau,\eta)\}$ belongs to a Donsker class
    \item $\hat{\tau}_{\text{s}}\stackrel{p}{\longrightarrow }\tau^{*}$ 
    \item Map $\tau\to \phi(x;\tau,\eta)$ is differentiable at $\tau^{*}$ uniformly in a neighrborhood of $\eta^{*}$
    \item The following matrix $\Omega = \mathrm{E}_{*}(\nabla_{\tau}\log q\nabla_{\tau^{\top}}\log q|_{\theta^{*}})$ is non-singular
\end{itemize}
to invoke Theorem 6.17. in \cite{BolthausenErwin2002LoPT}. Especially, the second condition is confirmed as follows. First, $\{q(x;\tau);\tau \in \Theta_{\tau}\}$ and $\{1(x\leq t),t\in \mathbb{R}\}$ belong to Donsker class from Example 19.16 and Example 19.18 in \citep{VaartA.W.vander1998As}. Then, noting that
\begin{align*}
    (q, \eta) \to \phi(q,\tau)
\end{align*}
is a Lipsthicz continuous function, from Example 19.20 in \citep{VaartA.W.vander1998As}, $\phi(q(x);\eta(x))=\phi(x;\tau,\eta)$ is also a Donsker class.  

We have 
\begin{align*}
&\sqrt{n}(\hat{\tau}_{\text{s}}-\tau^{*})=\Omega^{-1}\mathbb{G}_{n}\left\{\nabla_{\tau}\log q(x;\tau)|_{\tau^{*}}\right\} +\mathrm{o}_{p}(1),\, \\
&\sqrt{n}(\hat{\tau}_{\text{s}}-\tau^{*})\stackrel{d}{\rightarrow}\mathcal{N}(0,\Omega^{-1}). 
\end{align*}

The estimator $\hat{\tau}_{\text{s}}$ is considered as the one satisfying $\mathbb{P}_{n}\phi(x;\tau,\eta)|_{\hat{\tau}_{\text{s}},\hat{\eta}_{n}}=0$, where $w(x)=q(x;\tau)/\eta(x)$ and $\phi(x;\tau,\eta) $ is 
\begin{align*}
\left ( \left [f'\{h_{1}(w)\}-f'\{h_{2}(w)\}\right]h_{1}'(w)-f''\{h_{2}(w)\}h'_{2}(w)\left[h_{1}(w)-h_{2}(w)\right]\right )w\nabla_{\tau} \log q(x;\tau).
\end{align*}

From Theorem 6.17. in \cite{BolthausenErwin2002LoPT} based on the above assumptions, we have
\begin{align}
\label{eq:vanvan}
\sqrt{n}(\hat{\tau}_{\text{s}}-\tau^{*})=-V_{\tau^{*},\eta^{*}}^{-1}\sqrt{n}\mathrm{E}_{*}\{\phi(x)|_{\tau^{*},\hat{\eta}_{n}}\}-V_{\tau^{*},\eta^{*}}^{-1}\mathbb{G}_{n} \phi(x)|_{\tau^{*},\eta^{*}}+\mathrm{o}_{p}(1+\sqrt{n}\|\mathrm{E}_{*}[\phi(x)|_{\tau^{*},\hat{\eta}_{n}}]\|),
\end{align}
where $V_{\tau^{*},\eta^{*}}$ is a derivative of $\tau\to \mathrm{E}_{*}\{\phi(x;\tau,\eta^{*})\}$ at $\tau^{*}$. First, we calculate the derivative $V_{\tau^{*},\eta^{*}}$.
The derivative is 
\begin{align*}
 &\nabla_{\tau^{\top}}\mathrm{E}_{*}\{\phi(x;\tau,\eta^{*})|_{\tau^{*}}\}= \mathrm{E}_{*}\{\nabla_{\tau^{\top}}\phi(x;\tau,\eta^{*})|_{\tau^{*}}\}\\
 &=\sqrt{n}\mathrm{E}_{*}[f''\{h_{1}(w)\}h'_{1}(w)- f''\{h_{2}(w)\}h'_{2}(w)]h'_{1}(w)-  f''\{h_{2}(w)\}h'_{2}(w)\{h'_{1}(w)-h'_{2}(w)\}\\
& w\nabla_{\tau}\log q(x;\tau) \{\nabla_{\tau^{\top}}{w}\}
\{\hat{\eta}_{n}(x)-\eta^{*}(x)\}|_{\tau^{*},\eta^{*}} \\
 &= f''(1)\{h'_{1}(1)-h'_{2}(1)\}^{2}\mathrm{E}_{*}(\nabla_{\tau}\log q\nabla_{\tau^{\top}}\log q|_{\tau^{*}}) \\
 &= f''(1)\{h'_{1}(1)-h'_{2}(1)\}^{2}\Omega.
\end{align*}

Next consider each term in \eqref{eq:vanvan}. The second term in \eqref{eq:vanvan} vanishes because $\phi(x)|_{\tau^{*},\eta^{*}}$ is $0$. Therefore, we only analyze the first term in \eqref{eq:vanvan}:

\begin{align}
& \sqrt{n}\mathrm{E}_{*}\{\phi(x)|_{\tau^{*},\hat{\eta}_{n}}\}\nonumber =\sqrt{n}\mathrm{E}_{*}\{\phi(x)|_{\tau^{*},\hat{\eta}_{n}}\} -\sqrt{n}\mathrm{E}_{*}\{\phi(x)|_{\tau^{*},\eta^{*}}\} \nonumber \\
&=\sqrt{n}\mathrm{E}_{*}[ \nabla_{\eta}\phi(x)|_{\tau^{*},\eta^{*}}\{\hat{\eta}_{n}(x)-\eta^{*}(x)\}] \label{eq:first_term3} \\
&+\sqrt{n}\mathrm{E}_{*}\{\phi(x)|_{\tau^{*},\hat{\eta}_{n}}\} -\sqrt{n}\mathrm{E}_{*}\{\phi(x)|_{\tau^{*},\eta^{*}}\} -\sqrt{n}\mathrm{E}_{*}[ \nabla_{\eta}\phi(x)|_{\tau^{*},\eta^{*}}\{\hat{\eta}_{n}(x)-\eta^{*}(x)\}]  \label{eq:first_term4}.
\end{align}

We decompose $\sqrt{n}\mathrm{E}_{*}\{\phi(x)|_{\tau^{*},\hat{\eta}_{n}}\}$ into two terms again. The first term \eqref{eq:first_term3} is 
\begin{align*}
& \sqrt{n}\mathrm{E}_{*}[ \nabla_{\eta}\phi(x)|_{\tau^{*},\eta^{*}}\{\hat{\eta}_{n}(x)-\eta^{*}(x)\}] \\
=& \sqrt{n}\mathrm{E}_{*}[\{f''\{h_{1}(w)\}h'_{1}(w)- f''\{h_{2}(w)\}h'_{2}(w)\}h'_{1}(w)- f''\{h_{2}(w)\}h'_{2}(w)(h'_{1}(w)-h'_{2}(w)\}\{vgn\nabla_{\eta}{w}\} \\
& w\nabla_{\tau}\log q(x;\tau)|_{\tau^{*},\eta^{*}}(\hat{\eta}_{n}(x)-\eta^{*}(x))]  \\
=&-\sqrt{n}\int f''\left(1\right)\{h'_{1}(1)-h'_{2}(1)\}^{2} \frac{\nabla_{\tau}q(x;\tau)}{q(x;\tau)}|_{\tau^{*}}\{\hat{\eta}_{n}(x)-\eta^{*}(x)\}\mathrm{d}\mu(x) \\
=&-\sqrt{n}f''\left(1\right)\{h'_{1}(1)-h'_{2}(1)\}^{2}\mathbb{G}_{n}\left\{\frac{\nabla_{\tau}q(x;\tau)}{q(x;\tau)}|_{\tau^{*}}\right\}.\\
\end{align*}

In addition, the second residual term \eqref{eq:first_term4} vanishes because, for some large $C$ and $\tilde{\eta}$ is a between $\hat{\eta}_{n}$, we have
\begin{align*}
    &\|\sqrt{n}\mathrm{E}_{*}\{\phi_{\tau^{*},\hat{\eta}_{n}}\} -\sqrt{n}\mathrm{E}_{*}\{\phi_{\tau^{*},\eta^{*}}\}-\sqrt{n}\mathrm{E}_{*}[ \nabla_{\eta}\phi(x;\tau,\eta)_{|\tau^{*},\eta^{*}}\{\hat{\eta}_{n}(x)-\eta^{*}(x)\}]\| \\
    &=\sqrt{n}\|\mathrm{E}_{*}[\nabla_{\eta\eta}\phi(x;\tau,\eta)_{|\tau^{*},\tilde{\eta}}\{\hat{\eta}_{n}(x)-\eta^{*}(x)\}^{2}]\| \\
    &\leq C \sqrt{n}\|\mathrm{E}_{*}[\{\hat{\eta}_{n}(x)-\eta^{*}(x)\}^{2}].
\end{align*}
From the second line to the third line, we use an assumption (2b). The last term goes to $0$ in probability. By combining all things and substituting into \eqref{eq:vanvan}, the statement is proved. 
\end{proof}

\begin{proof}[Proof of  Corollary \ref{cor:score}]
The score function $S(x;\theta)$ can be written as 
\begin{align*}
\nabla_{\theta}\log p(x;\theta)-\int \nabla_{\theta}\log p(x;\theta) \frac{ p(x;\theta)}{\int p(x;\theta)\mathrm{d}\mu(x)}\mathrm{d}\mu(x).
\end{align*}
Fisher information matrix $\mathfrak{I}_{\theta^{*}}^{-1}$  is $\mathrm{Var}_{*}\{S(x;\theta)|_{\theta^{*}}\}$, that is,
\begin{align*}
    \mathrm{E}_{*}\{\nabla_{\theta}\log p(x;\theta)\nabla_{\theta^{\top}}\log p(x;\theta)|_{\theta^{*}}\}
    -\mathrm{E}_{*}\{\nabla_{\theta}\log p(x;\theta)|_{\theta^{*}}\}\mathrm{E}_{*}\{\nabla_{\theta^{\top}}\log p(x;\theta)|_{\theta^{*}}\}.
\end{align*}
On the other hand, the component corresponding $\theta^{*}$ in $\Omega^{-1}$ can be also written as 
\begin{align*}
\mathrm{E}_{*}\{\nabla_{\theta}\log p(x;\theta)\nabla_{\theta^{\top}}\log p(x;\theta)|_{\theta^{*}}\}
    -\mathrm{E}_{*}\{\nabla_{\theta}\log p(x;\theta)|_{\theta^{*}}\}\mathrm{E}_{*}\{\nabla_{\theta^{\top}}\log p(x;\theta)|_{\theta^{*}}\}
\end{align*}
from Theorem \ref{thm:4.2} and Woodbury formula.
This is the same as the Fisher information matrix. This concludes the proof. 
\end{proof}

\begin{proof}[Proof of  Theorem \ref{thm:4.4}]

We can do in the proof of Theorem \ref{thm:4.2}. Here, $\eta(x)$ belongs to a Donsker class from Example 19.24; therefore, $\phi(x;\tau,\eta)$ belongs to Donsker class from Lipschitz continuous property. The problem is a drift term. We can derive the given theorem by calculating the drift term in the same way.
The drift term $\sqrt{n}\mathrm{E}_{*}[\phi_{\tau^{*},\hat{\eta}_{n}}]$ is decomposed into two terms, the main term:

\begin{align*}
\sqrt{n}\int  \frac{\nabla_{\tau}q(x;\tau)}{q(x;\tau)}|_{\tau^{*}}\left\{\frac{1}{nh^{d_{x}}}\sum_{i=1}^{n}K\left(\frac{x-x_{i}}{h}\right)-\eta^{*}(x)\right\}\mathrm{d}\mu(x),
\end{align*}
and the residual term. The main term corresponds to the term \eqref{eq:first_term3} in Theorem \ref{thm:4.2} and the residual term corresponds to the term \eqref{eq:first_term4} in Theorem \ref{thm:4.2}. As revealed in the proof, the residual term is written as $\mathrm{O}_{p}(\sqrt{n}\| \hat{\eta}-\eta^{*}(x) \|^{2})$. This term is equal to the order $\mathrm{o}_{p}(1)$ because  $\frac{\nu}{2\nu+d_{x}}>1/4$ holds from the assumption (2e). 
Next, we have
\begin{align*}
    \sqrt{n}\int  \frac{\nabla_{\tau}q(x;\tau)}{q(x;\tau)}|_{\tau^{*}}\left\{\frac{1}{nh^{d_{x}}}\sum_{i=1}^{n}K\left(\frac{x-x_{i}}{h}\right)\mathrm{d}\mu(x)-\mathrm{d}\mathbb{P}_{n}(x)\right\} = \mathrm{o}_{p}(1).
\end{align*}

This holds from Theorem 8.11 in \cite{newey} using assumptions (2d) and (2e). 
Then, the drift term becomes
\begin{align*}
     \sqrt{n}\mathrm{E}_{*}[\phi_{\tau^{*},\hat{\eta}_{n}}]
     &= \sqrt{n}\int  \frac{\nabla_{\tau}q(x;\tau)}{q(x;\tau)}|_{\tau^{*}}\left\{\mathrm{d}\mathbb{P}_{n}(x)-\eta^{*}(x)\right\}\mathrm{d}\mu(x)+\mathrm{o}_{p}(1) \\
     &= \mathbb{G}_{n} \left\{ \frac{\nabla_{\tau}q(x;\tau)}{q(x;\tau)}|_{\tau^{*}}\right\}+\mathrm{o}_{p}(1).
\end{align*}
We have calculated the drift term. For the rest of the proof, it is the same as the proof in Theorem \ref{thm:4.2}.
\end{proof}

\begin{proof}[Proof of  Theorem \ref{thm:ps-gamma}]
We redefine $\sigma \equiv (\theta^{\top},c_{1},c_{2})^{\top}$. To avoid abuse of notations, we write $U_{\alpha,\beta}(x;\sigma)$ as $U(x)$.

As in the proof of Theorem \ref{thm:4.2}, we have
\begin{align}
\sqrt{n}(\hat{\sigma}_{\text{ns-}\gamma}-\sigma^{*}) 
=-V_{\sigma^{*},\eta^{*}}^{-1}\sqrt{n}\mathrm{E}_{*}\{U(x)|_{\sigma^{*},\hat{\eta}_{n}}\}-V_{\sigma^{*},\eta^{*}}^{-1}\mathbb{G}_{n} U(x)|_{\sigma^{*},\eta^{*}} 
+\mathrm{o}_{p}(1+\sqrt{n}\|\mathrm{E}_{*}\{U(x)|_{\sigma^{*},\hat{\eta}_{n}}\}\|),
\label{eq:tktk}
\end{align}
where $\hat{\sigma}_{\text{ns-}\gamma}$ is a solution to $\mathrm{\tilde{E}}[U(x;\sigma)]=0$ and  $V_{\sigma^{*},\eta^{*}}$ is a derivative of the map $\sigma \to \mathrm{E}_{*}\{U(x;\sigma,\eta^{*})\}$ at $\sigma^{*}$. 

First, we calculate the derivative $V_{\sigma^{*},\eta^{*}}$.
This becomes

\begin{align*}
    \mathrm{E}_{*}\left[\begin{pmatrix}
    (\beta-\alpha)\nabla_{\theta} s(x;\theta)s(x;\theta)^{\top} & -s(x;\theta)& s(x;\theta) \\ 
    (\beta-1)s(x;\theta)^{\top}\exp(c_{1}) & \exp(c_{1})& 0 \\
    (\alpha-1)s(x;\theta)^{\top}\exp(c_{2}) & 0 &\exp(c_{2}). \\
    \end{pmatrix}
    \right],
\end{align*}

which is evaluated at $\sigma^{*}$ and $s(x;\theta)=\nabla_{\theta}\log p(x;\theta)$.
The term corresponding $\theta$ in the above matrix $V_{\sigma^{*},\eta^{*}}^{-1}$ is 
\begin{align*}
    &(\beta-\alpha)^{-1}\left[\mathrm{E}_{*}\{\nabla_{\theta}\log p(x;\theta)\nabla_{\theta^{\top}} \log p(x;\theta)\}- 
    \mathrm{E}_{*}\{\nabla_{\theta}\log p(x;\theta)\}\mathrm{E}_{*}\{\nabla_{\theta^{\top}} \log p(x;\theta)\}\right]^{-1}|_{\theta^{*}}\\
    &=(\beta-\alpha)^{-1}\mathfrak{J}_{\theta^{*}}^{-1}.
\end{align*}
Then, we analyze each term in \eqref{eq:tktk}. First of all, the second term in \eqref{eq:tktk} becomes zero because $U(x;\sigma^{*},\eta^{*})=0$. Therefore, we only consider the first term in \eqref{eq:tktk}. We have

\begin{align*}
\sqrt{n}\mathrm{E}_{*}\{U(x)|_{\sigma^{*},\hat{\eta}_{n}}\}
&=
\sqrt{n}\mathrm{E}_{*}[ \nabla_{\eta}U(x)_{|\sigma^{*},\eta^{*}}\{\hat{\eta}_{n}(x)-\eta^{*}(x)\}]+\mathrm{o}_{p}(1)\\
&=\sqrt{n}\mathrm{E}_{*}\left[
\begin{pmatrix}
(\alpha-\beta)\frac{\nabla_{{\theta}}\log p(x;\theta)}{\eta^{*}(x)}  \\
(\beta-1)/\eta^{*}(x) \\
(\alpha-1)/\eta^{*}(x) \\
\end{pmatrix}
\{\hat{\eta}_{n}(x)-\eta^{*}(x)\}\right]+\mathrm{o}_{p}(1).
\end{align*}

Therefore, the first term corresponding $\theta$ in the above equation becomes
\begin{align*}
&\mathfrak{J}_{\theta^{*}}^{-1}
 \sqrt{n}\int \nabla_{\theta}\log p(x;\theta)|_{\theta^{*}}\{\hat{\eta}_{n}(x)-\eta^{*}(x)\}\mathrm{d}\mu(x) \\
 &=\mathfrak{J}_{\theta^{*}}^{-1}\mathbb{G}_{n}\left\{\nabla_{\theta}\log p(x;\theta)|_{\theta^{*}}\right \}.
\end{align*}

Finally, we get
\begin{align*}
\sqrt{n}(\hat{\theta}_{\text{ns-}\gamma}-\theta^{*})=\mathfrak{J}_{\theta^{*}}^{-1}\mathbb{G}_{n}\{\nabla_{\theta}\log p(x;\theta)|_{\theta^{*}}\}+\mathrm{o}_{p}(1).
\end{align*}
\end{proof}

\begin{proof}[of Theorem \ref{thm:concavity}]
Let us define $\ell_i(\tau)$ as the loss for the sample $x_i$, i.e., 
\begin{align*}
\ell_i(\tau)= -f'(z_i)+w_i f'(z_i)-f(z_i),
\end{align*}
where $z_i=q(x_i;\tau)/\hat{\eta}(x_i)$. 
The loss function is expressed by the total sum of $\ell_i(\tau)$ over all samples. 
For the unnormalized exponential model, 
some calculation yields the Hessian matrix of $\ell_i(\tau)$, 

\begin{align*}
  \nabla^2 \ell_i(\tau)  &=
 \left[
 f''(z_i)z_i^2+\left(z_i-1\right)\{f'''(z_i)z_i^2+f''(z_i)z_i\}
\right]
 \phi(x_i)\phi(x_i)^{\top}\\
 &= z_i\left\{(2z_i-1)f''(z_i) + z_i(z_i-1)f'''(z_{i})\right\} \phi(x_i)\phi(x_i)^{\top}.
\end{align*}

The assumption of the theorem guarantees that the coefficient above is non-negative; hence, the Hessian matrix of $\ell_i(\tau)$ is non-negative definite, so is the
loss function. Eventually, the loss function is convex in the parameter $\tau$. 
\end{proof}

 We prove Theorem \ref{thm:5.1}. Before that, we show Lemm \ref{lem:mis}.

\begin{lemma}
\label{lem:mis}
Under certain regularity conditions, we have
\begin{align}
\label{eq:sep}
&\sqrt{n}(\hat{\tau}_{\text{s}}-\tau^{*})=\Omega_{1m}^{-1}\mathbb{G}_{n}\left\{\nabla_{\tau}\log q(x;\tau)|_{\tau^{*}}\right\} +\mathrm{o}_{p}(1),\\
&\sqrt{n}(\hat{\tau}_{\text{s}}-\tau^{*})\stackrel{d}{\rightarrow} \mathcal{N}(0,\Omega_{1m}^{-1}\Omega_{2m}\Omega_{1m}^{-1}),
\end{align}
where $\tau^{*}=(c^{*},\theta^{*})$ is a value such that 
\begin{align*}
    \exp(c^{*})=\int p(x;\theta^{*})\mathrm{d}\mu(x),\,0 =\mathrm{E}_{*}\{S(x;\theta)\},
\end{align*}
and
\begin{align*}
\Omega_{1m}&=-\mathrm{E}_{*}\left[\left\{1-\frac{q(x;\tau)}{\eta^{*}(x)}\right\}\nabla_{\tau^{\top}}\nabla_{\tau}\log q(x;\tau)|_{\tau^{*}}\right]
+\mathrm{E}_{*}\left\{\frac{q(x;\tau)}{\eta^{*}(x)}\nabla_{\tau}\log q(x;\tau)\nabla_{\tau^{\top}}\log q(x;\tau)|_{\tau^{*}}\right\}, \\
\Omega_{2m}&=\mathrm{Var}_{*}\left\{\nabla_{\tau}\log q(x;\tau)|_{\tau^{*}}\right\}.
\end{align*}
\end{lemma}

\begin{proof}[Proof of  Lemma \ref{lem:mis}]
The estimator $\hat{\tau}_{\text{s}}$ can be considered as the one satisfying $\mathbb{P}_{n}\phi_{\hat{\tau}_{\text{s}},\hat{\eta}_{n}}=0$, where
\begin{align*}
    \phi(x;\tau,\eta)=\nabla_{\tau}\log q(x;\tau)-\left\{\frac{q(x;\tau)}{\eta(x)}\right\}\nabla_{\tau}\log q(x;\tau).
\end{align*}
From Theorem 6.17. \cite{BolthausenErwin2002LoPT}, we have
\begin{align}
\label{eq:vanvan3}
\sqrt{n}(\hat{\tau}_{\text{s}}-\tau^{*})
=-V_{\tau^{*},\eta^{*}}^{-1}\sqrt{n}\mathrm{E}_{*}\{\phi|_{\tau^{*},\hat{\eta}_{n}}\}-V_{\tau^{*},\eta^{*}}^{-1}\mathbb{G}_{n} \phi(x_{i};\tau^{*},\eta^{*}) +\mathrm{o}_{p}(1+\sqrt{n}\|\mathrm{E}_{*}[\phi|_{\tau^{*},\hat{\eta}_{n}}]\|), 
\end{align}
where $V_{\tau^{*},\eta^{*}}$ is a derivative of $\tau\to \mathrm{E}_{*}\{\phi(x;\tau,\eta^{*})\}$ at $\eta^{*}$. 

First, we will see a more specific form of $\tau^{*}$. The value $\tau^{*}$ satisfy the equation $\mathrm{E}_{*}\{\phi(x;\tau,\eta^{*})\}=0$. Noting that $\nabla_{\tau^{\top}}\log q(x;\tau)=(1,\nabla_{\theta^{\top}}\log p(x;\theta))$, we can get the form of $c^{*}$ and $\theta^{*}$ specified in the statement. 

Next, we calculate the derivative $V_{\tau^{*},\eta^{*}}$.
The derivative is 
\begin{align*}
 &\nabla_{\tau^{\top}}\mathrm{E}_{*}\{\phi(x;\tau,\eta)|_{\tau^{*}}\} \\
 &= \mathrm{E}_{*}\{\nabla_{\tau^{\top}}\phi(z;\tau,\eta)|_{\tau^{*}}\}\\
 &= \mathrm{E}_{*}\left[\left\{-1+\frac{q(x;\tau)}{\eta(x)}\right\}\nabla_{\tau^{\top}}\nabla_{\tau}\log q(x;\tau)|_{\tau^{*}}\right]-\mathrm{E}_{*}\left\{\frac{q(x;\tau)}{\eta(x)}\nabla_{\tau}\log q(x;\tau)\nabla_{\tau^{\top}}\log q(x;\tau)\right\}|_{\tau^{*}} \\
 &=-\Omega_{1}.
\end{align*}

Next, consider each term in \eqref{eq:vanvan3}. The second term is 
\begin{align*}
   \Omega_{1}^{-1}\mathbb{G}_{n}\left[\left\{1-\frac{q(x;\tau)}{\eta(x)}\right\}|_{\tau^{*},\eta^{*}}\nabla_{\tau}\log q(x;\tau)|_{\tau^{*}}\right].
\end{align*}
The first term is 

\begin{align*}
\sqrt{n}\Omega_{1m}^{-1}\mathrm{E}_{*}(\phi|_{\tau^{*},\hat{\eta}_{n}}) &=\sqrt{n}\Omega_{1m}^{-1}\mathrm{E}_{*}(\phi|_{\tau^{*},\hat{\eta}_{n}}) -\sqrt{n}\mathrm{E}_{*}(\phi|_{\tau^{*},\eta^{*}})  \nonumber \\
&=\sqrt{n}\Omega_{1m}^{-1}\mathrm{E}_{*}\left[ \nabla_{\eta}\phi(x)_{|\tau^{*},\eta^{*}}\{\hat{\eta}_{n}(x)-\eta^{*}(x)\}\right] +\mathrm{o}_{p}(1)\\
&=\sqrt{n}\Omega_{1m}^{-1}\mathrm{E}_{*}\left[
\frac{q(x;\tau)}{\eta^{2}(x)}
\nabla_{\tau}\log q(x;\tau)|_{\tau^{*},\eta^{*}}
\{\hat{\eta}_{n}(x)-\eta^{*}(x)\}\right]+\mathrm{o}_{p}(1)\\
&=\sqrt{n}\Omega_{1m}^{-1}\int  \frac{q(x;\tau)}{\eta(x)}|_{\tau^{*},\eta^{*}}\nabla_{\tau}\log q(x;\tau)\{\hat{\eta}_{n}(x)-\eta^{*}(x)\}\mathrm{d}\mu(x)
+\mathrm{o}_{p}(1) \\
&= \Omega_{1m}^{-1}\mathbb{G}_{n}\left\{\frac{q(x;\tau)}{\eta(x)}|_{\tau^{*},\eta^{*}}\nabla_{\tau}\log q(x;\tau)|_{\tau^{*}}\right\}+\mathrm{o}_{p}(1).
\end{align*}

Adding the first term and the second term, we get
\begin{align*}
\sqrt{n}(\hat{\tau}_{\text{s}}-\theta^{*})=\Omega_{1m}^{-1}\mathbb{G}_{n}\left\{\nabla_{\tau}\log q(x;\tau)|_{\tau^{*},\eta^{*}}\right\}+\mathrm{o}_{p}(1).
\end{align*}
Therefore, we conclude that $\sqrt{n}(\hat{\tau}_{\text{s}}-\theta^{*})$ converges to the normal distribution $\mathcal{N}(0,\Omega_{1m}^{-1}\Omega_{2m}\Omega_{1m}^{-1})$. 
\end{proof}

\begin{proof}[Proof of  Theorem \ref{thm:5.1}]

We calculate matrix, corresponding to $\theta$ term in Lemma \ref{lem:mis}. The matrix $\Omega_{1m}$ in Lemma \ref{lem:mis} is equal to the following block matrix:
\begin{align*}
\begin{bmatrix}
\Omega_{11} & \Omega_{12} \\
\Omega_{21} & \Omega_{22}
\end{bmatrix},
\end{align*}
where $\Omega_{11}=1$,
\begin{align*}
\Omega_{21}=\mathrm{E}_{*}\left\{\nabla_{\tau}\log q(x;\tau)\frac{q(x;\tau)}{\eta^{*}(x)}|_{\tau^{*}}\right\}
\end{align*}
and 
\begin{align*}
\Omega_{22} = \mathrm{E}_{*}\left[\left\{1-\frac{q(x;\tau)}{\eta^{*}(x)}\right\}\nabla_{\theta^{\top}}\nabla_{\theta}\log q(x;\tau)|_{\tau^{*}}\right] 
+\mathrm{E}_{*}\left\{\frac{q(x;\tau)}{\eta^{*}(x)}\nabla_{\theta}\log q(x;\tau)\nabla_{\theta^{\top}}\log q(x;\tau)|_{\tau^{*}}\right\}.
\end{align*}
From Woodbury formula, the corresponding term to $\theta$ in $\Omega_{1m}^{-1}$ is $\Omega^{\dagger}_{1m}$ where 

\begin{align*}
\Omega^{\dagger}_{1m} 
&=\mathrm{E}_{*}\left[\left\{1-\frac{q(x;\tau)}{\eta^{*}(x)}\right\}\nabla_{\theta^{\top}}\nabla_{\theta} \log p(x;\theta)|_{\tau^{*}}\right]+
\mathrm{E}_{*}\left\{\frac{q(x;\tau)}{\eta^{*}(x)}\nabla_{\theta}\log p(x;\theta)\nabla_{\theta^{\top}}\log p(x;\theta)|_{\tau^{*}}\right\}\\
&-\mathrm{E}_{*}\left\{\frac{q(x;\tau)}{\eta^{*}(x)}\nabla_{\theta}\log p(x;\theta)|_{\tau^{*}}\right \}
\mathrm{E}_{*}\left\{\frac{p(x;\theta)}{\eta^{*}(x)}\nabla_{\theta^{\top}}\log p(x;\theta)|_{\tau^{*}} \right\}.
\end{align*}

On the other hand, the corresponding part in $\Omega_{2m}$ is $\Omega^{\dagger}_{2m}$, where \begin{align*}
    \Omega^{\dagger}_{2m}=\mathrm{Var}\{\nabla_{\theta}\log p(x;\theta)|_{\theta^{*}}\},
\end{align*}
noting that $\nabla_{\tau^{\top}} \log q(x;\tau)=\{1,\nabla_{\theta^{\top}} \log p(x;\theta)\}$.
This concludes the proof. 
\end{proof}

Note the difference between the normalized case and unnormalized case is that $\tilde{p}(x;\theta^{*})/\eta^{*}$ is used when the model is normalized; while,  $q(x;\tau^{*})/\eta^{*}$ is used in $\Omega^{\dagger}_{1m}$ and $\Omega^{\dagger}_{2m}$ when the model is unnormalized.

\end{document}